\documentclass[twoside]{article}
\usepackage[accepted]{aistats2019}
\usepackage[usenames,dvipsnames,dvinames]{xcolor}
\usepackage{multirow}
\usepackage{microtype}
\usepackage{enumitem}
\usepackage{amsmath,amsthm}
\usepackage[mathscr]{eucal}
\usepackage{amssymb}
\usepackage{footnote}
\usepackage{graphicx}
\usepackage{caption}
\usepackage{subfloat}
\usepackage[vlined,ruled]{algorithm2e}
\usepackage{algorithmic}
\usepackage{wrapfig}
\usepackage{float}
\usepackage{caption}

\DeclareMathOperator*{\argmax}{argmax}
\DeclareMathOperator*{\argmin}{argmin}
\newtheorem{definition}{Definition}
\newtheorem{theorem}{Theorem}

\newtheorem{lemma}[theorem]{Lemma}

\usepackage{ifthen}
\newboolean{isdraft}
\setboolean{isdraft}{false} 
\ifthenelse{\boolean{isdraft}}{%
\newcommand{\myaddcomment}[3]{{\color{#1}{\ensuremath{\langle\!\!\langle}{\bf {#2} :} {#3}\ensuremath{\rangle\!\!\rangle}}}}
\newcommand{\rishabh}[1]{\myaddcomment{orange}{Rishabh}{#1}}
\newcommand{\JTR}[1]{\myaddcomment{orange}{Jeff\ensuremath{\rightarrow}Rishabh}{#1}}
\newcommand{\jeff}[1]{\myaddcomment{blue}{Jeff}{#1}}
\newcommand{\RTJ}[1]{\myaddcomment{blue}{Rishabh\ensuremath{\rightarrow}Jeff}{#1}}
}{
\newcommand{\rishabh}[1]{}
\newcommand{\JTR}[1]{}
\newcommand{\jeff}[1]{}
\newcommand{\RTJ}[1]{}
\newcommand{\toboth}[1]{}
}

\usepackage[bookmarks, colorlinks=true, citecolor=Violet,linkcolor=Mahogany,urlcolor=blue]{hyperref}
\begin{document}

\twocolumn[
\aistatstitle{Near Optimal Algorithms for Hard Submodular Programs with Discounted Cooperative Costs}

\aistatsauthor{ Rishabh Iyer \And Jeff Bilmes }

\aistatsaddress{ Microsoft Corporation \And  University of Washington, Seattle}
]

\begin{abstract}
  In this paper, we investigate a class of submodular problems which in general are very hard. These include minimizing a
  submodular cost function under combinatorial constraints, which
  include cuts, matchings, paths, etc., optimizing a submodular
  function under submodular cover and submodular knapsack
  constraints, and minimizing a ratio of submodular functions. All these problems appear in several real world problems but have hardness factors of
  $\Omega(\sqrt{n})$ for general submodular cost functions. We show
  how we can achieve constant approximation factors when we restrict
  the cost functions to low rank sums of concave over modular
  functions. A wide variety of machine
  learning applications are very naturally modeled via this subclass
  of submodular functions. Our work therefore provides a tighter connection
  between theory and practice by enabling theoretically satisfying
  guarantees for a rich class of expressible, natural, and useful submodular
  cost models. We empirically demonstrate the utility of our models on
  real world problems of cooperative image matching and sensor
  placement with cooperative costs.\looseness-1
\end{abstract}

\section{Introduction} \label{sec:introduction}

Submodular functions provide a rich class of expressible models for a
variety of machine learning problems. Submodular functions occur
naturally in two flavors. In minimization problems, they model
notions of cooperation, attractive potentials, and economies of scale,
while in maximization problems, they model aspects of coverage,
diversity, and information. A set function $f: 2^V \to \mathbb R$ over
a finite set $V = \{1, 2, \ldots, n\}$ is \emph{submodular} 
\cite{fujishige2005submodular} if for all
subsets $S, T \subseteq V$, it holds that $f(S) + f(T) \geq f(S \cup
T) + f(S \cap T)$. Given a set $S \subseteq V$, we define the
\emph{gain} of an element $j \notin S$ in the context $S$ as $f(j | S)
\triangleq f(S \cup j) - f(S)$. A perhaps more intuitive
characterization of submodularity is as follows:
a function $f$ is submodular if it satisfies
\emph{diminishing marginal returns}, namely $f(j | S) \geq f(j | T)$
for all $S \subseteq T, j \notin T$, and is \emph{monotone} if $f(j |
S) \geq 0$ for all $j \notin S, S \subseteq V$.

In this paper, we address the following a family of hard submodular optimization problems. The first one is constrained submodular minimization~\cite{goel2009optimal,iwata2009submodular,curvaturemin,rkiyersemiframework2013,jegelka2011-inference-gen-graph-cuts,svitkina2008submodular}:
\begin{align*}
\mbox{Problem 1: } \min\{f(X) | X \in \mathcal C\}
\end{align*}
where the function $f$ is \emph{monotone submodular}, and $\mathcal C$ is a \emph{combinatorial constraint}, which could represent a cardinality lower bound constraint, or more complicated ones like cuts, matchings, trees, or paths in a graph. With cut constraints, this problem becomes cooperative cuts~\cite{jegelka2011-nonsubmod-vision}, and with matching constraints, we call this \emph{cooperative matchings}, which we introduce and utilize in this paper.\looseness-1

The second problem asks for minimizing a monotone submodular cost function $f$, while simultaneously maximizing a monotone submodular coverage function $g$. A natural way to model this bi-optimization problem is to introduce one of $f$ and $g$ as a constraint~\cite{nipssubcons2013}. In particular, we obtain two optimization problems:
\begin{align*}
\label{eqn:scsc_scsk_def}
\mbox{Problem 2: }& \min \{f(X) \, | \,g(X) \geq c\}, \\
\mbox{Problem 3: }& \max \{ g(X) \,| \,f(X) \leq b\},
\end{align*}

The fourth problem considered in this paper is minimizing the ratio of submodular functions~\cite{bai2016algorithms}.
\begin{align*}
\mbox{Problem 4: } \min\{f(X)/g(X) \,| \, \emptyset \subset X \subset V\}
\end{align*}

A key assumption in this paper is that the functions $f$ and $g$ in Problems 1-4, are monotone submodular -- an assumption that, as we shall see, is natural in many applications. Problem 2 is a special case of Problem 1, with $\mathcal C = \{X: g(X) \geq c\}$. Furthermore, Problem 2 and Problem 3 are closely related and, loosely speaking, duals of each other~\cite{nipssubcons2013}. Similarly, Problem 4 is closely related to Problems 2 and 3, in that given an approximation algorithm for Problems 2 or 3, we can obtain an approximation with similar guarantees for Problem 4~\cite{bai2016algorithms} (also considered in~\cite{Qian2017OptimizingRO} with general monotone set functions). Problem 1 is constrained submodular minimization, while Problems 2, 3 and 4 try to simultaneously minimize one submodular function while maximizing another.

Problems 1-4 appear naturally in several machine learning applications. However, in the worst case all four problems have polynomial hardness factors of $\Omega(\sqrt{n})$~\cite{svitkina2008submodular,rkiyersemiframework2013,nipssubcons2013,bai2016algorithms}. An important observation is that the
polynomial hardness of problems 1 - 3, comes up mainly due to the
submodular cost function $f$ -- they do not depend as much on the
constraints $\mathcal C$ or the submodular function
$g$~\cite{rkiyersemiframework2013,nipssubcons2013}. In the case of Problem 4, the hardness depends on both $f$ and $g$~\cite{bai2016algorithms}.
\begin{figure}
\includegraphics[width=0.42\textwidth]{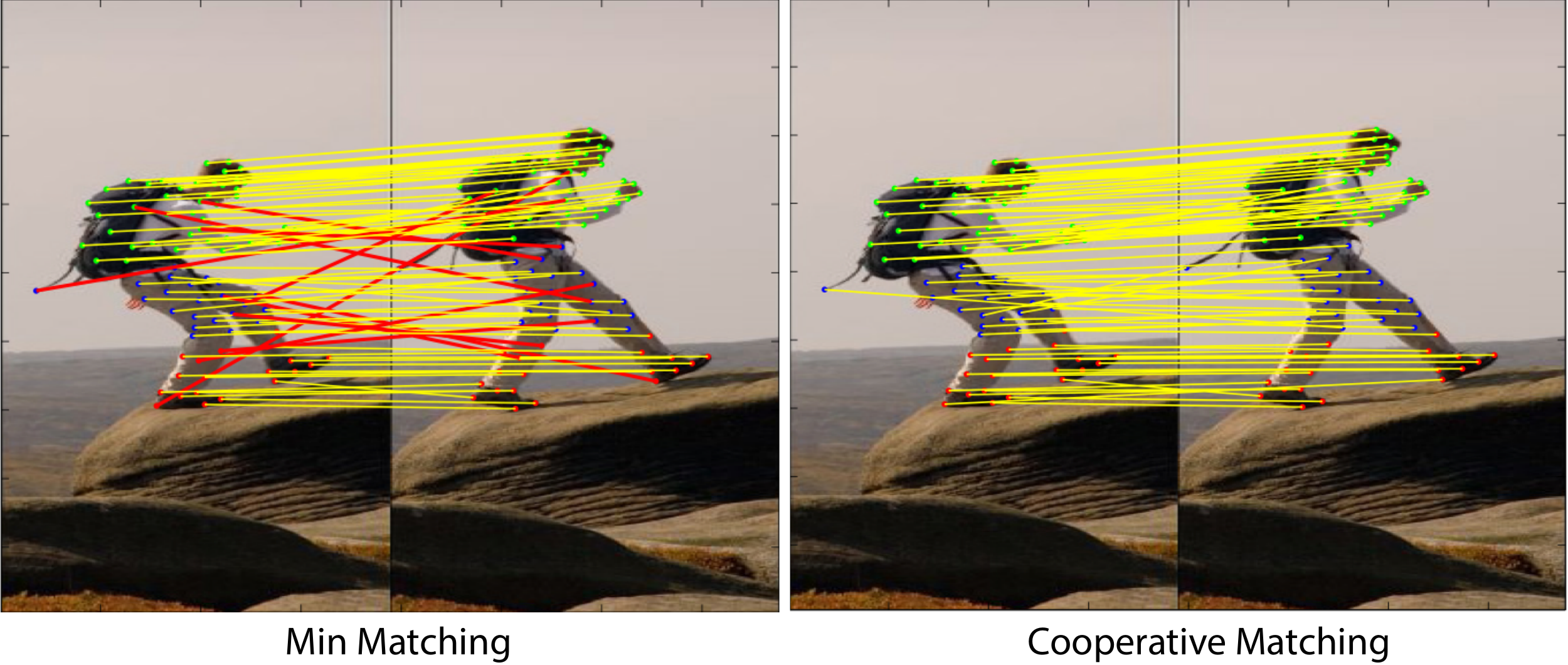}
\caption{
An illustration of cooperative matchings. The goal is to match
corresponding points between the two images. On the left,
bipartite matching is used and there are many mismatched points (indicated
by red edges). On the right, similar points are clustered and offered
a within-cluster discount via a submodular function, significantly
reducing the number of mis-matched points.}
\label{fig:illustration}
\end{figure}

On the other hand, these problems come up as models in many machine
learning applications. These lower bounds are specific to rather
contrived classes of functions, whereas much better results can be
achieved for many practically relevant cases. The pessimistic worst
case results are somewhat discouraging, d begging the need to
quantify sub-classes of submodular functions that are more amenable to these optimization problems. Only limited past work has focused on investigating these problems with potentially a subclass of submodular functions. \cite{nips2013extendedvcurv,nipssubcons2013,bai2016algorithms} provide bounds for Problems 1-4 based on the notion of \emph{curvature}, and argue how several submodular functions (e.g. clustered concave over modular functions) have bounded curvature. Their curvature bounds depend on the choice of the submodular functions, and in certain cases yield no improvement over the worst case bounds. For classes of functions with bounded curvature, their bounds yield improved results.

In this paper, we
focus on a tractable yet expressive subclass of submodular cost
functions $f$, namely low rank sums of concave over modular functions. 
\begin{definition}
\underline{Low rank} sums of concave over modular functions are the class of functions representable as $f(X) = \sum_{i = 1}^k \psi_i(w_i(X))$, where $\psi_i$s are monotone concave, and $k$ is constant or $O(\log n)$. 
\end{definition}
\emph{Low Rank} in this context means that the number of components in
the sum is small (i.e., $k$ is small).  Our use of the terminology ``low
rank'' is identical to that used in \cite{goyal2013fptas}.  We argue
how this subclass naturally models many interesting applications
of problems 1 - 4 in machine learning. We do not
need to consider the entire class of submodular functions (which
includes rather contrived instances), but only this subclass. This
observation helps us in providing better connections between theory
and practice. The main specialty of this subclass is that these
functions effectively model cooperation between objects via discounts
provided by concave functions. Moreover, we show that this subclass
admits fully polynomial time approximation schemes for Problem 1, and
constant factor approximation guarantees for Problems 2 and
3.  Similarly, we achieve constant factor approximation guarantees for Problem 4, when $f$ is a low rank sum of concave over Modular functions, and $g$ is an arbitrary submodular function, a significant improvement over~\cite{bai2016algorithms}. The bounds we obtain are significantly better than the worst case bounds, and also an improvement over the bounds achieved using the curvature~\cite{rkiyersemiframework2013,nipssubcons2013}.
\looseness-1

Low rank sums of concave over modular functions in Problems 1 - 4, fit as natural models in several machine learning problems. Below, we summarize some of these.

\textbf{Image segmentation (Cooperative Cuts):} Markov random fields
with pairwise attractive potentials occur naturally in modeling image
segmentation and related
applications~\cite{boykov2004experimental}. While models are tractably
solved using graph-cuts, they suffer from the shrinking bias problem,
and images with elongated edges are not segmented properly.
When modeled via a submodular
function, however, the cost of a cut is not just the sum of the edge weights,
but a richer function that allows cooperation between edges, and
yields superior results on many challenging
tasks
(see,
for example, the results of the image
segmentations in~\cite{jegelka2011-nonsubmod-vision}).
This was
achieved in~\cite{jegelka2011-nonsubmod-vision} by partitioning the
set of edges $\mathcal E$ of the grid graph into groups of similar
edges (or types) $\mathcal E_1, \cdots, \mathcal E_k$, and defining a
function $f(S) = \sum_{i = 1}^k \psi_i(w(S \cap \mathcal E_i)), S
\subseteq \mathcal E$, where $\psi_i$s are concave functions and $w$
encodes the edge potentials. This ensures that we offer a
\emph{discount} to edges of the same type. Moreover, the number of
types of edges are typically much smaller than the number of pixels,
so this is a low-rank sum of concave
functions. 


 \textbf{Image Correspondence (Cooperative Matchings): } The simplest model for matching key-points in pairs of images (which is also called the correspondence problem) can be posed as a bipartite matching. These models, however, do not capture interaction between the pixels. We illustrate the difficulty of this in Figure~\ref{fig:illustration}.  One kind of desirable interaction is that similar or neighboring pixels be matched together. We can achieve this as follows. First we cluster the key-points in the two images into $k$ groups (this is illustrated in Figure~\ref{fig:illustration}-left via green, blue and red key-points).  This induces a clustering of edges that can be given a discount via a submodular function (details are given in Section~\ref{sec:coop-image-match}).
In practice, the number of groups ($k$) can be much smaller than $n$
and this is a low-rank sum of concave over modular
functions. Figure~\ref{fig:illustration}-right shows how the
submodular matchings improves over the simple bipartite matching. In
particular, the minimum matching approach produces many spurious
matches between clusters (shown in red) that are avoided via the
cooperation described above.

\textbf{Sensor Placement or Feature Selection: }Often, the problem of
choosing sensor locations $A$ from a given set of possible locations
$V$ can be modeled~\cite{krause2008near,rkiyeruai2012} by maximizing
the mutual information between the chosen variables $A$ and the
unchosen set $V \backslash A$ (i.e.,
$g(A) = I(X_A; X_{V \backslash A})$). Alternatively, we may wish to
maximize the mutual information between a set of chosen sensors $X_A$
and a quantity of interest $C$ (i.e., $g(A) = I(X_A ; C)$) assuming
that the set of features $X_A$ are conditionally independent given
$C$~\cite{krause2008near}. Both these functions are submodular. Since
there are costs involved, we want to simultaneously minimize the cost
$f(A)$. Often this cost is
submodular~\cite{krause2008near,rkiyeruai2012}, since there is
typically a discount when purchasing sensors in bulk (or computing
features), and we can express this via Problems 2 and 3. For example,
there may be diminished cost for placing a sensor in a particular
location given placement in certain other locations. Similarly,
certain features might be cheaper to use given that others are already
being computed (e.g., those that use an FFT). A natural cost model in
such cases is $f(A) = \sum_{i=1}^k \psi_i( m(A \cap S_i) )$ where
$\psi_i$'s are concave, $m(j)$ is the cost of sensor (or feature) $j$
and $S_1, \cdots, S_k$ are groups of similar sensors or features.
Typically, $k$ is much smaller than $n$ and this can be expressed as
low rank sum of concave over modular functions.

\section{Background \& Existing Algorithms}
\label{sec:backgr-exist-algor}

The basic idea for most combinatorial algorithms solving Problems 1 - 4, are based on approximating the cost function $f$ with a tractable surrogate function $\hat{f}$~\cite{goel2009optimal,goemans2009approximating,jegelka2011-inference-gen-graph-cuts,curvaturemin,nipssubcons2013,rkiyersemiframework2013,bai2016algorithms}. Moreover, all four problems have similar guarantees. We characterize the quality of the solution via the notion of approximation factors. In particular, we say that an algorithm achieves an approximation factor of $\alpha \geq 1$ for Problem 1, if we can obtain a set $\hat{X}$ such that $f(\hat{X}) \leq \alpha f(X^*)$, where $X^*$ is the optimizer of Problem 1. For Problems 2 and 3, we use the notion of bi-criterion approximation factors. An algorithm is a $[\sigma, \rho]$ bi-criterion algorithm for Problem 2 if it is guaranteed to obtain a set $\hat{X}$ such that $f(\hat{X}) \leq \sigma f(X^*)$ (approximate optimality) and $g(\hat{X}) \geq \rho c$ (approximate feasibility), where $X^*$ is an optimizer of Problem 2.  Typically, $\sigma \geq 1$ and $\rho \leq 1$.  Similarly, an algorithm is a $[\rho,\sigma]$ bi-criterion algorithm for Problem 3 if it is guaranteed to obtain a set $\hat{X}$ such that $g(\hat{X}) \geq \rho g(X^*)$ and $f(\hat{X}) \leq \sigma b$, where $X^*$ is the optimizer of Problem 3. 
Moreover, problems 2 and 3 are very closely related~\cite{nipssubcons2013}, in that an approximation algorithm for one problem can be used to obtain guarantees for the other problem. The two problems also have matching hardness factors. For Problem 4, we study an algorithm which achieve $\alpha$-approximation guarantees, in that we can achieve a set $\hat{X}$ such that $h(\hat{X}) \leq \alpha h(X^*)$ where $h(X) = f(X)/g(X)$ and $X^*$ is the optimal minimizer of $h$. \looseness-1

\textbf{Supergradient based Algorithm (SGA): }One such method uses the supergradients of a submodular function~\cite{rkiyersemiframework2013,curvaturemin,goel2009optimal,jegelka2011-nonsubmod-vision,rkiyersubmodBregman2012} to obtain modular upper bounds in an iterative manner. In particular, define a modular upper bound:
\begin{align*}
m^f_{X}(Y) \triangleq f(X) - \!\!\!\! \sum_{j \in X \backslash Y } f(j| V \backslash j) + \!\!\!\! \sum_{j \in Y \backslash X} f(j| X) \geq f(Y)
\end{align*} 
The algorithm starts with the $X^0 = \emptyset$ and sequentially sets $X^{i+1}$ as the solution of the corresponding problem (1, 2 or 3) with a surrogate function as $\hat{f}(X) = m^f_{X^i}(X)$~\cite{rkiyersemiframework2013,jegelka2011-inference-gen-graph-cuts,nipssubcons2013}. In each case, this subproblem is much easier. For example, in the case of Problem 1, the subproblem becomes, 
\begin{align}
X^{i+1} = \min\{m^f_{X^i}(X) | X \in \mathcal C\},
\end{align} 
which is a linear cost problem, poly-time solvable for many constraints, like cardinality, cuts, matchings, paths etc. 

In the case of Problems 2 and 3, these subproblems are 
\begin{align*}
X^{i+1} = \min\{m^f_{X^i}(X) | g(X) \geq c\} \mbox{ and } \\
X^{i+1} = \max\{g(X) | m^f_{X^i}(X) \leq b\},
\end{align*}
which are the submodular set cover and the submodular knapsack problems respectively~\cite{wolsey1982analysis,nemhauser78,nipssubcons2013}, and are constant factor approximable to a factor of $1 - 1/e$. 

With Problem 4, the subproblem becomes,
\begin{align*}
X^{i+1} = \min\{m^f_{X^i}(X)/g(X) \,|,\ \emptyset \subset X \subset V\}
\end{align*}
This can be approximated up to a factor of $e/(e-1)$ via a Greedy algorithm~\cite{bai2016algorithms}.

\begin{lemma}
Define $\alpha_f(X^*) = \frac{|X^*|}{1 + (|X^*| - 1)(1 - \hat{\kappa_f}(X^*))} \leq \min\{|X^*|, \frac{1}{1 - \hat{\kappa_f}(X^*)}\}$, where $\hat{\kappa_f}(X) = 1 - \frac{\sum_{j \in X} f(j | X \backslash j)}{\sum_{j \in X} f(j)}$ represents the average curvature of the function $f$.
The supergradient based iterative algorithm (SGA) achieves an approximation factor of $\alpha_f(X^*)$ for Problem 1, and bicriteria factors satisfying $\sigma = \alpha_f(X^*)$ and $\rho = 1 - 1/e$ for Problems 2 and 3. Finally, SGA achieves an approximation factor of $e/(e-1)*\alpha_f(X^*)$ for Problem 4.
\end{lemma}
This Lemma follows easily from the results in~\cite{rkiyersemiframework2013,curvaturemin,nipssubcons2013,bai2016algorithms}. We can also achieve a non-bicriteria approximation factor for Problem 2, which is worse than the bicriteria factor by a $\log$ factor~\cite{nipssubcons2013}. A key quantity which defines the approximation factor above is the average curvature $\hat{\kappa_f}(X^*)$, which in turn depends on the concave functions. If the concave function is $\psi_i(x) = x^a, a \in (0, 1)$, SGA admits approximation factors of $O(|X^*|^{1-a})$~\cite{curvaturemin}. On the other hand, if the concave function is $\psi(x) = \log(1+ x)$, the guarantees are $O(|X^*|)$, which is much poorer. 

The supergradient based algorithm is easy to implement, and also works well in practice~\cite{rkiyersemiframework2013,jegelka2011-nonsubmod-vision}. For the general class of submodular functions, these results are close to the optimal bounds, and are, in fact, tight for some constraints. Nevertheless, the worst case guarantees seem discouraging, particularly for the class of low rank sums of concave over modular functions that we consider here, and that as mentioned above are natural 
for many applications.

\textbf{Ellipsoidal Approximation based Algorithm (EA):} Another generic approximation of a submodular function, introduced by Goemans et.\  al~\cite{goemans2009approximating}, is based on approximating the submodular polyhedron by an ellipsoid. The main result states that any polymatroid (monotone submodular) function $f$,
can be approximated by a function of the form $\sqrt{w^f(X)}$ for a
certain modular weight vector $w^f \in \mathbb R^V$, such that $\sqrt{w^f(X)} \leq f(X) \leq
O(\sqrt{n}\log{n}) \sqrt{w^f(X)}, \forall X \subseteq V$.
A simple trick then provides a curvature-dependent approximation~\cite{curvaturemin}. We have the following result borrowed from~\cite{nipssubcons2013,curvaturemin,bai2016algorithms}.
\begin{lemma}
Define $\alpha = O(\frac{\sqrt{n}}{1 + (\sqrt{n} - 1)(1 - \kappa_f}))$, where $\kappa_f = 1 - \frac{\min_{j \in X} f(j | V \backslash j)}{\min_{j \in X} f(j)}$ represents the worst case curvature of the function $f$.. The Ellipsoidal Approximation based algorithm (EA) achieves an approximation factor of $\alpha$  for Problem 1, and bicriteria factors satisfying $\sigma = \alpha$ and $\rho = 1 - 1/e$ for Problems 2 and 3. Similarly EA achieves an approximation guarantee of $e\alpha/(e-1)$ for Problem 4.
\end{lemma}
The Ellipsoidal Approximation obtains the tightest bounds for Problems 1-4~\cite{nipssubcons2013,curvaturemin,goemans2009approximating,goel2009optimal,jegelka2011-inference-gen-graph-cuts,bai2016algorithms}. This is again for the general class of submodular functions and the worst case factor of $O(\sqrt{n})$ is quite discouraging. This algorithm, however is very expensive computationally, and is not practical for solving machine learning applications~\cite{rkiyersemiframework2013}.

\section{Improved Algorithms for Low-rank sums of concave-modular functions} 
\label{sec:impr-algor-low}

\begin{figure}
\centering
\includegraphics[width = 0.35\textwidth]{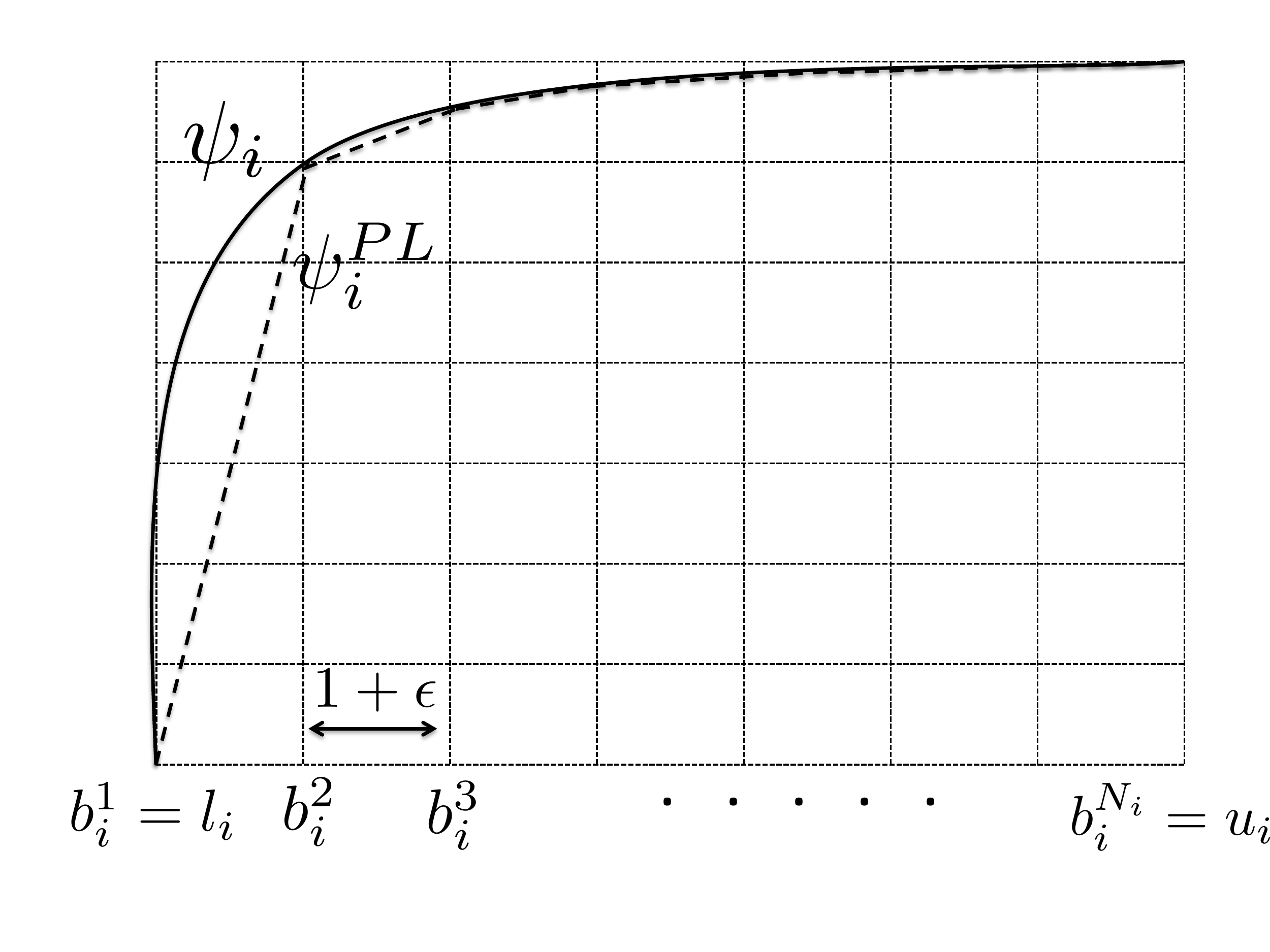}
\caption{Visualizing $\psi^{PL}_i$ and $\psi_i$.}
\label{fig1}
\end{figure}

Our main new results are that we can achieve a fully polynomial
time approximation scheme for Problem 1, and constant factor
approximation guarantees for Problems 2, 3 and 4 when the cost function
$f$ is a low rank sum of concave over modular functions
(Theorem~\ref{PLAguarantee}). Our
techniques build on recent methods used for minimizing quasi-concave
functions over solvable polytopes~\cite{nikolova2010approximation,mittal2013fptas,goyal2013fptas,kelner2007hardness}.

Assume the concave functions $\psi_i$'s are monotone functions, i.e., $\psi_i(y) \leq \psi_i(y^{\prime}), \forall y \leq y^{\prime}$. We also assume that for all $i$, $\psi_i(ky) \leq k^c \psi_i(y)$ for $k \geq 1, y \geq 0$ and some constant $c$. The second assumption holds for a number of concave functions, including $\psi_i(x) = x^a, a \in (0, 1)$, $\psi_i(x) = \log(1+ x)$ and $\psi_i(x) = \min(x, a)$. \looseness-1

\begin{figure}
\vspace{-2.9ex}
\includegraphics[width = 0.35\textwidth]{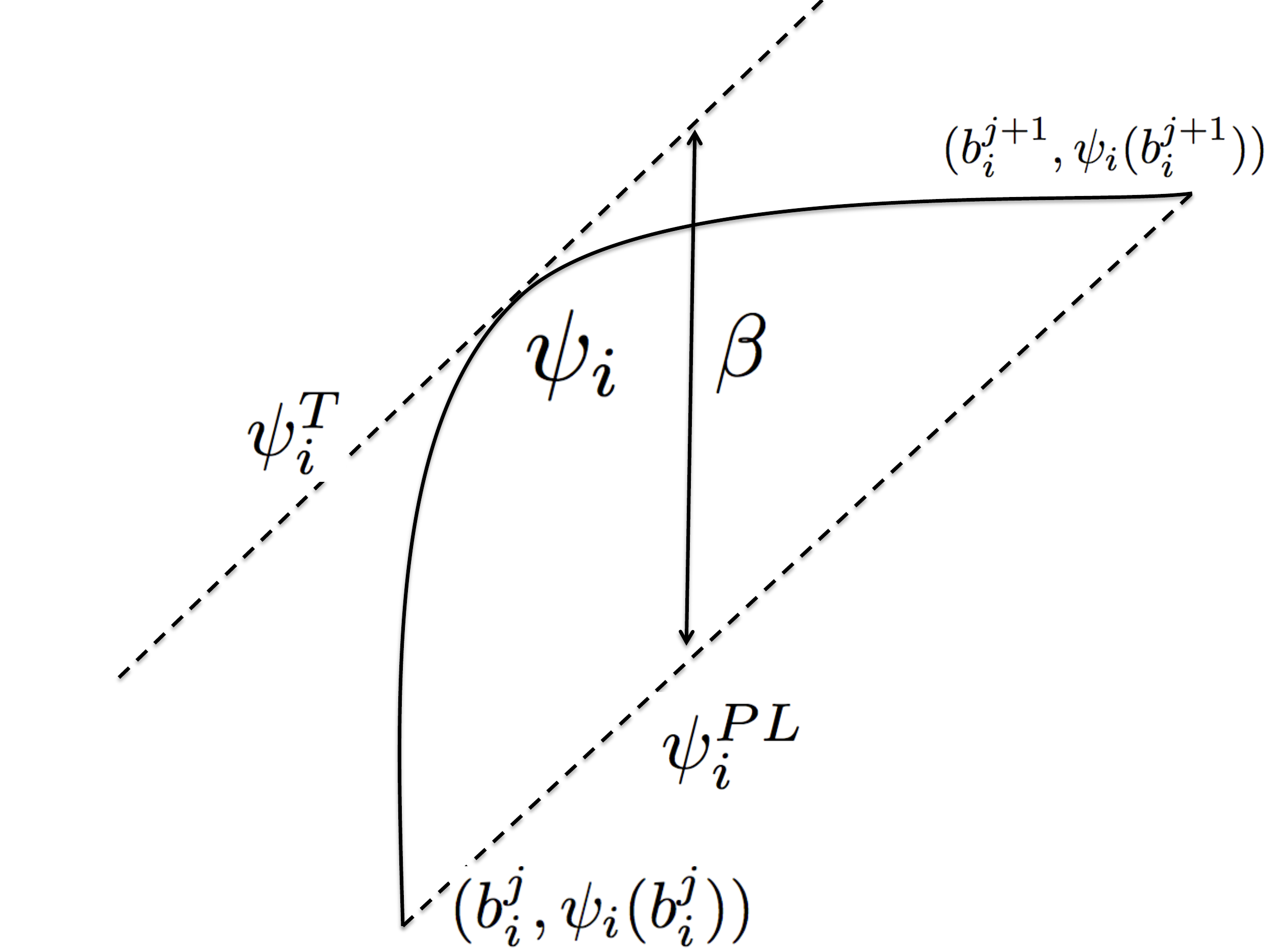}
\vspace{-2ex}
\caption{Showing $\psi_i^T$ and $\psi_i^{PL}$.}
\label{fig2}
\end{figure}

The main idea of this approach is to replace the concave functions $\psi_i$'s by piece-wise linear approximations $\psi_i^{PL}(x)$. We define an approximation of $f(X)$ as $f^{PL}(X)$ defined as $f^{PL}(X) = \sum_{i = 1}^k \psi_i^{PL}(w_i(X))$. We then optimize this piece-wise linear approximation function, and the approximation factor comes based on the tightness of this piece-wise linear approximation. We call this procedure the \emph{piece-wise linear approximation based algorithm} (PLA).

We compute this approximation as follows. In the case of Problem 1, compute $l_i = \min\{w_i(X) | X \in \mathcal C\}$ and $u_i = \max\{w_i(X) | X \in \mathcal C\}$ for each $i = 1, 2, \cdots, k$. Both these computations are linear cost problems and are polynomial time for most constraints. In case these are NP hard for Problem 1, or in the case of Problems 2, 3 and 4, we set $l_i = \min\{w_i(j), j \in V: w_i(j) > 0\}$ and $u_i = w_i(V)$. Then divide the range $[l_i, u_i]$ into pieces with breakpoints $b_i^1, b_i^2, \cdots, b_i^{N_i}$ such that $b_i^1 = l_i$, $b_i^2 = l_i (1 + \epsilon)$, $b_i^3 = l_i (1 + \epsilon)^2$ and so on, for any $\epsilon > 0$. It is easy to see that $N_i = \log_{1+\epsilon} u_i/l_i \approx \log(u_i/l_i)/\epsilon$. The precision $\epsilon$ defines the fineness of the points, and the quality of the approximation.

For all $i = 1, 2, \cdots, k$, define the piece-wise linear function $\psi^{PL}_i$, via the breakpoints $b_i^1, b_i^2, \cdots, b_i^{N_i}$. A visualization of this is shown in Figure~\ref{fig1}, where the dotted lines are the piece-wise approximation, while the solid curve is the concave function $\psi_i$. We first show that the function $f^{PL}$ approximates the function $f$ within a factor of $1 + \epsilon$.

\begin{lemma}
The piece-wise linear function $f^{PL}$ defined with a precision $\epsilon^{\prime}$ satisfies,
\begin{align}
f^{PL}(X) \leq f(X) \leq (1 + \epsilon^{\prime})^c f^{PL}(X) = (1 + \epsilon) f^{PL}(X)
\end{align}
where $c$ is a constant such that $\psi_i(ky) \leq k^c \psi_i(y)$ for $k \geq 1, y \geq 0$ for all $i = 1, 2, \cdots, k$.
\end{lemma}
\begin{proof}
By the construction of $f^{PL}$, and the concavity of the $\psi_i$s, it is easy to see that $f^{PL}(X) \leq f(X), \forall X \subseteq V$. To show the upper bound, consider a region defined by breakpoints $b_i^j$ and $b_i^{j+1}$. Due to concavity of $\psi_i$, there exists a tangent at some point in $[b_i^j, b_i^{j+1}]$ whose slope equals that of the line connecting $(b_i^j, \psi_i(b_i^j))$ and $(b_i^{j+1}, \psi_i(b_i^{j+1}))$. This tangent line upper bounds the concave function $\psi_i$, and we can denote the corresponding upper bound as $\psi_i^T$. It then holds that $\psi_i^{PL}(y) \leq \psi_i(y) \leq \psi_i^T(y)$. We now show that $\psi_i^T(y) \leq (1 + \epsilon^{\prime})^c \psi_i^{PL}(y)$.  

We now focus on the region $[b_i^j, b_i^{j+1}]$. Let $\beta$ be the constant difference between the two (parallel) lines, in terms of the $y$ value. A visualization of this is shown in Figure~\ref{fig2}. We would like to give a worst case bound on $\psi^T_i(y)/\psi^{PL}_i(y), \forall y \in [b_i^j, b_i^{j+1}]$. Notice that $\psi^T_i(y)/\psi^{PL}_i(y) = 1 + \beta/\psi^{PL}_i(y) \leq 1 + \beta/\psi^{PL}_i(b^j_i) = 1 + \beta/\psi_i(b^j_i) \leq \psi_i(b^{j+1}_i)/\psi_i(b^j_i)$. The last inequality holds since $\psi_i(b^j_i) + \beta \leq \psi_i(b^{j+1}_i)$, and the second last one holds since $b^j_i$ is a break point.

Moreover, $b^{j+1}_i = b^j_i (1 + \epsilon^{\prime})$ and hence $\psi_i(b^{j+1}_i)/\psi_i(b^j_i) \leq \psi_i((1 + \epsilon^{\prime})b^j_i)/\psi_i(b^j_i) \leq (1 + \epsilon^{\prime})^c = 1 + \epsilon$.
\end{proof}

We now show how we can exactly solve Problems 1, 2 and 3 using the
cost function $f^{PL}$.  Let $s^j_i$ denote the slopes of the
piece-wise linear functions -- in other words, $s^j_i =
[\psi_i(b^{j+1}_i) - \psi_i(b^j_i)]/[b^{j+1}_i - b^j_i]$. Also, we
denote $c^j_i$ as the corresponding intercepts. The functions
$\psi_i^{PL}$ are characterized by the pairs $\{(s^1_i, c^1_i),
(s^2_i, c^2_i), \cdots, s^{N_i}_i, c^{N_i}_i)\}$, and $\psi_i^{PL}(y)
= s^j_i.y + c^j_i, \forall y \in [b^j_i, b^{j+1}_i]$. We then 
consider the $\prod_{i = 1}^k N_i$ different possibilities of the
cross-terms. Define $J = [j_1, j_2, \cdots, j_k]$ as a vector such that
$J \in [1, N_1] \times [1, N_2] \times \cdots \times [1, N_k]$.

In the case of Problem 1, PLA solves a set of optimization problems,
\begin{align}\label{pla1}
\hat{X}_J = &\argmin\{\sum_{i = 1}^k s^{j_i}_i w_i(X) + c^{j_i}_i \,|\, X \in \mathcal C\}, \nonumber \\ &\forall J \in [1, N_1] \times [1, N_2] \times \cdots \times [1, N_k].
\end{align}
The final solution $\hat{X}$ is the minimum among the ones above. For problem 2, we consider the set of problems,
\begin{align}\label{pla2}
\hat{X}_J = &\argmin\{\sum_{i = 1}^k s^{j_i}_i w_i(X) + c^{j_i}_i \,|\, g(X) \geq c\}, \nonumber \\ &\forall J \in [1, N_1] \times [1, N_2] \times \cdots \times [1, N_k],
\end{align}
and again set $\hat{X}$ is the minimum among the $\hat{X}_J$'s above. Similarly, for Problem 3, we solve,
\begin{align}\label{pla3}
\hat{X}_J = &\argmax\{g(X) \,|\, \sum_{i = 1}^k s^{j_i}_i w_i(X) + c^{j_i}_i \leq b\}, \nonumber \\
&\forall J \in [1, N_1] \times [1, N_2] \times \cdots \times [1, N_k]
\end{align}
We set $\hat{X}$ corresponding to the set with the largest value of $g(\hat{X}_J)$.
Finally, for Problem 4, we have:
\begin{align}\label{pla4}
\hat{X}_J = &\argmin\{\frac{\sum_{i = 1}^k s^{j_i}_i w_i(X) + c^{j_i}_i}{g(X)} \}, \nonumber \\ &\forall J \in [1, N_1] \times [1, N_2] \times \cdots \times [1, N_k].
\end{align}

Our main result is that these simple procedures provide improved guarantees for all three problems.
\begin{theorem}\label{PLAguarantee}
PLA achieves an approximation factor of $1+\epsilon$ for Problem 1 as long as a linear function can be exactly minimized under $\mathcal C$. PLA also achieves a bi-criterion approximation factor satisfying $\sigma = 1 + \epsilon$ and $\rho = 1 - 1/e$ for Problems 2 and 3. PLA also achieves a non bicriterion approximation factor of $(1 + \epsilon) \log g(V)$ for Problem 2. PLA also achieves an approximation factor of $e(1 + \epsilon)/(e-1)$ for Problem 4. The worst case complexity of PLA is $\prod_{i = 1}^k \log(u_i/l_i) (\frac{1}{\epsilon})^k T = O((\frac{1}{\epsilon})^k T)$, where $T$ is the complexity of Problems 1-4, with a linear cost function $f$.
\end{theorem}
\begin{proof}
We first show that PLA solves Problems 1-4 with the surrogate function $f^{PL}$. Note that with the piece-wise linear approximation, Problem 1 becomes $\min_{X \in \mathcal C} \sum_{i = 1}^k \psi^{PL}_i(w_i(X)) = \min_{X \in \mathcal C} \sum_{i = 1}^k \min\{s^{1}_i w_i(X) + c^{1}_i, s^{2}_i w_i(X) + c^{2}_i, \cdots, s^{N_i}_i w_i(X) + c^{N_i}_i\}$. This holds since $\psi^{PL}_i(y) = \min\{s^{1}_i y + c^{1}_i, s^{2}_i y + c^{2}_i, \cdots, s^{N_i}_i y + c^{N_i}_i\}$, due to the concavity of $\psi_i$'s. We can then rewrite this as $\sum_{i = 1}^k \psi^{PL}_i(w_i(X)) = \min_{J \in \mathcal J} \sum_{i = 1}^k s^{j_i}_i w_i(X) + c^{j_i}_i$, where $\mathcal J = [1, N_1] \times [1, N_2] \times \cdots \times [1, N_k]$. Combining these facts, we can rewrite the problem as $\min_{X \in \mathcal C} \min_{J \in \mathcal J} \sum_{i = 1}^k s^{j_i}_i w_i(X) + c^{j_i}_i$, which after interchanging the $\min$'s becomes $\min_{J \in \mathcal J} \min_{X \in \mathcal C} \sum_{i = 1}^k s^{j_i}_i w_i(X) + c^{j_i}_i$. This is exactly Eq.~\eqref{pla1}.

The algorithm for Problem 2 (Eq.~\eqref{pla2}) is basically the same as that of Problem 1, since it is a special case. Similarly we can write Problem 4 as $\min_{X \in \mathcal C} \min_{J \in \mathcal J} \sum_{i = 1}^k (s^{j_i}_i w_i(X) + c^{j_i}_i)/g(X)$ which is equivalent to $\min_{J \in \mathcal J} \min_{X \in \mathcal C} \sum_{i = 1}^k (s^{j_i}_i w_i(X) + c^{j_i}_i)/g(X)$, which becomes equation~\eqref{pla4}. Equations~\eqref{pla1}, \eqref{pla2} and~\eqref{pla3} each become instances of Problems 1, 2 and 4 with $f$ being modular and the approximation guarantees follow directly from~\cite{nips2013extendedvcurv,nipssubcons2013,bai2016algorithms}.

To deal with Problem 3, we use the fact that $f^{PL}(X) = \min_{J \in \mathcal J} \sum_{i = 1}^k (s^{j_i}_i w_i(X) + c^{j_i}_i)$, and hence we have the constraint, $\{\min_{J \in \mathcal J} \sum_{i = 1}^k s^{j_i}_i w_i(X) + c^{j_i}_i \leq b\}$. First we show that $X_J$ is feasible for all $ \in \mathcal J$. This follows easily from the fact that if for any $X$, $w_J(X) \leq b$, it holds that $\min_{J \in \mathcal J} w_J(X) \leq b\}$. Next, let $X^*$ be the optimal solution of Problem 3, and let $J^*$ be such that $\sum_{i = 1}^k s^{j^*_i}_i w_i(X^*) + c^{j^*_i}_i = f^{PL}(X^*)$. Note that our algorithm covers $J^*$ and hence $g(\hat{X}) \geq g(\hat{X}_{J^*}) \geq (1 - 1/e)g(X^*)$, where $1 - 1/e$ is the approximation factor of the submodular knapsack problem~\cite{sviridenko2004note}. Note that the approximation factor of Problem 1 with $f^{PL}$ is $1$ assuming $\mathcal C$ admits an exact solution with linear cost functions, while the factor for problem $2$ is $\log g(V)$ for non-bicriterion algorithms, and a bicriterion factor of $[1, 1 - 1/e]$ with a bi-criterion algorithm~\cite{wolsey1982analysis,nipssubcons2013}.
\end{proof}

Results similar to Theorem~\ref{PLAguarantee} have been shown for a generalization of Problem 1, which asks for constrained optimization of low rank functions~\cite{nikolova2010approximation,mittal2013fptas,goyal2013fptas,kelner2007hardness,kohliOJ13,fujishige1999minimizing}. This problem in general is not a combinatorial optimization problem. However, when the functions are quasi-concave, the optimum lies on an extreme point, and hence, can be posed as a combinatorial optimization problem. Problem 1 asks for optimizing a specific subclass of concave (and hence quasi-concave) functions.  \cite{nikolova2010approximation,goyal2013fptas,kelner2007hardness} focus on the class of low rank quasi-concave functions, while \cite{mittal2013fptas} consider the general class of low rank functions. While their algorithms apply to our class of functions as well, their approach while being more general, is also more complicated and involved.  \cite{kohliOJ13} also consider a special case of Problem 1, with $\mathcal C$ being the family of cuts (i.e., the cooperative cut problem). Interestingly, they suggest an algorithm that is identical to PLA when $f$ is a (low rank) sum of truncations 
(i.e., $\psi_i(x) = \min(x, a)$).
For general sums of low-rank concave functions, they resort to the algorithms of~\cite{mittal2013fptas, goyal2013fptas}. We provide a generic algorithm, which not only works for a much large class of constraints and functions, but also extends to the Problems 2, 3 and 4. Moreover, it is easy to see that our algorithms would also work for the more general problem of minimizing low rank sums of concave functions, over a solvable polytope.\looseness-1

Note that the complexity of PLA is polynomial in $\frac{1}{\epsilon}$, but exponential in $k$. Hence this makes sense only if $k$ is a constant or is $O(\log n)$. If $k$ is a constant (with respect to $n$), PLA is a fully polynomial time approximation scheme (FPTAS)~\cite{vazirani2004approximation}. If $k = O(\log n)$, then PLA is a polynomial-time approximation scheme (PTAS). This assumption is reasonable for many of the applications of Problems 1-4 (see details of this in the experiments section). Moreover, there are a number of ways one can speed up PLA. A very simple observation is that PLA is amenable to a distributive implementation via Map-Reduce. In particular, let $N = O(\frac{1}{\epsilon}^k)$ denote the total number of computations of PLA (i.e., this is the number of times one performs an instance of Problems 1-4 with a modular function). All these can be performed in parallel on $m$ processing
systems. We output the best from each system to a central processor, which finds the optimal amongst these. The complexity of this distributive procedure is $O(NT/m + m)$, (where $T$ is the complexity of using a modular function in the place of $f$ in Problems 1-4), which improves the overall complexity by a factor of $m$.\JTR{check this, as originally
$m$ was not defined. Also, $T$ is currently not defined but should be.}

In addition, we can also provide early stopping criterion and heuristics for speeding up PLA. One strategy of implementing PLA, is to start with $j_i = 1, \forall i = 1, 2, \cdots, k$, and incrementally increase $j_i$ in a coordinate ascent fashion. The following lemma gives a sufficient condition for stopping PLA.
\begin{lemma}
Let $J = [j_1, j_2, \cdots, j_k]$ be such that the corresponding solution $\hat{X}_J$ satisfies $w_i(X_J) \in [b_i^{j_i}, b_i^{j_i + 1}], \forall i$. Then $\hat{X}_J$ is the (near) optimal solution for Problems 1, 2 and 3.
\end{lemma}
The values of $w_i(X_J)$ also suggest the direction of the co-ordinate wise algorithm. For example, if $w_i(X_J) < b_i^{j_i}$, it suggests that the value of $j_i$ be decreased. Similarly, if $w_i(X_J) > b_i^{j_i + 1}$, its a sign that $j_i$ be decreased. In this manner, one can define a greedy like heuristic to implement PLA~\cite{kohliOJ13}, which picks for every coordinate, the slope which increases the objective value the most. Many of these heuristics have been considered in \cite{kohliOJ13} in the case of cuts, and when the function class is low rank sums of truncations. These heuristics are all polynomial in $k$, but are not guaranteed to obtain the optimal solutions. Moreover, in certain cases (for example, the case of cuts), one can do parametric versions, thereby solving a set of related problems simultaneously~\cite{kohliOJ13,fujishige1999minimizing}.

\section{Experiments}
\label{sec:experiments}

We next experimentally evaluate the performance of our methods. The utility of the constrained minimization algorithms for cooperative cuts have been investigated in~\cite{kohliOJ13}. In this paper, we consider the applications of cooperative image matching and sensor placement.

\subsection{Cooperative Image matching}
\label{sec:coop-image-match}

\begin{figure*}
        \centering              \includegraphics[width=0.1\textwidth]{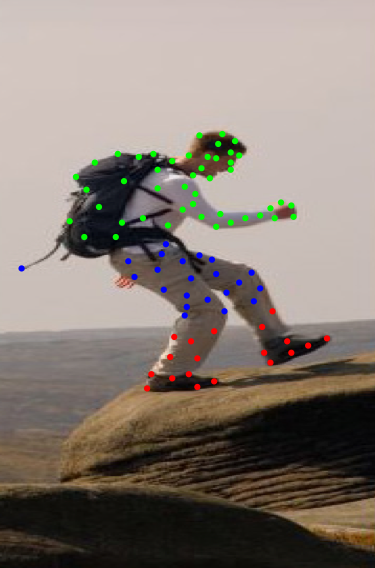}
        ~                 \includegraphics[width=0.2\textwidth]{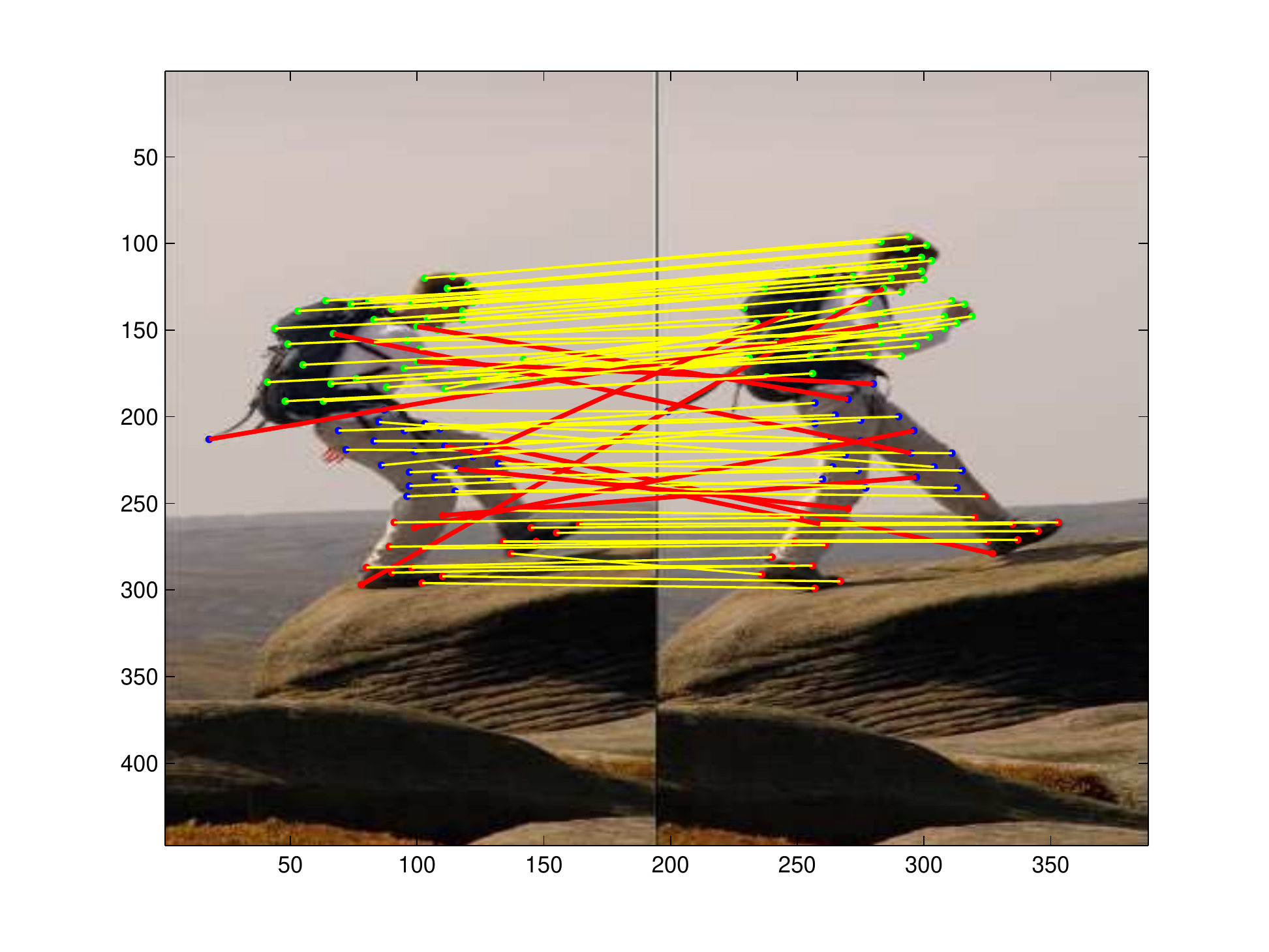}
        ~                 \includegraphics[width=0.2\textwidth]{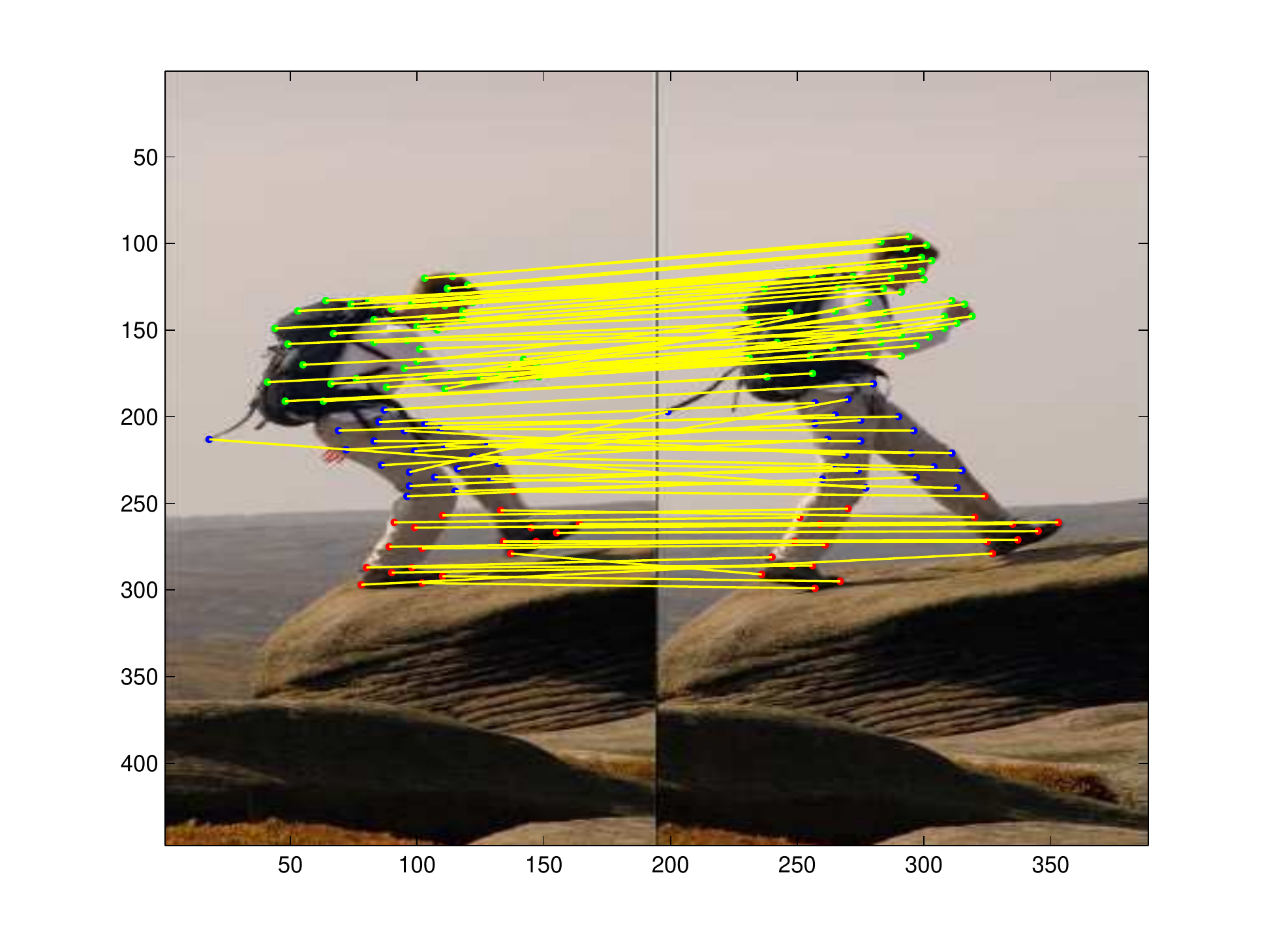}
        ~               \includegraphics[width=0.2\textwidth]{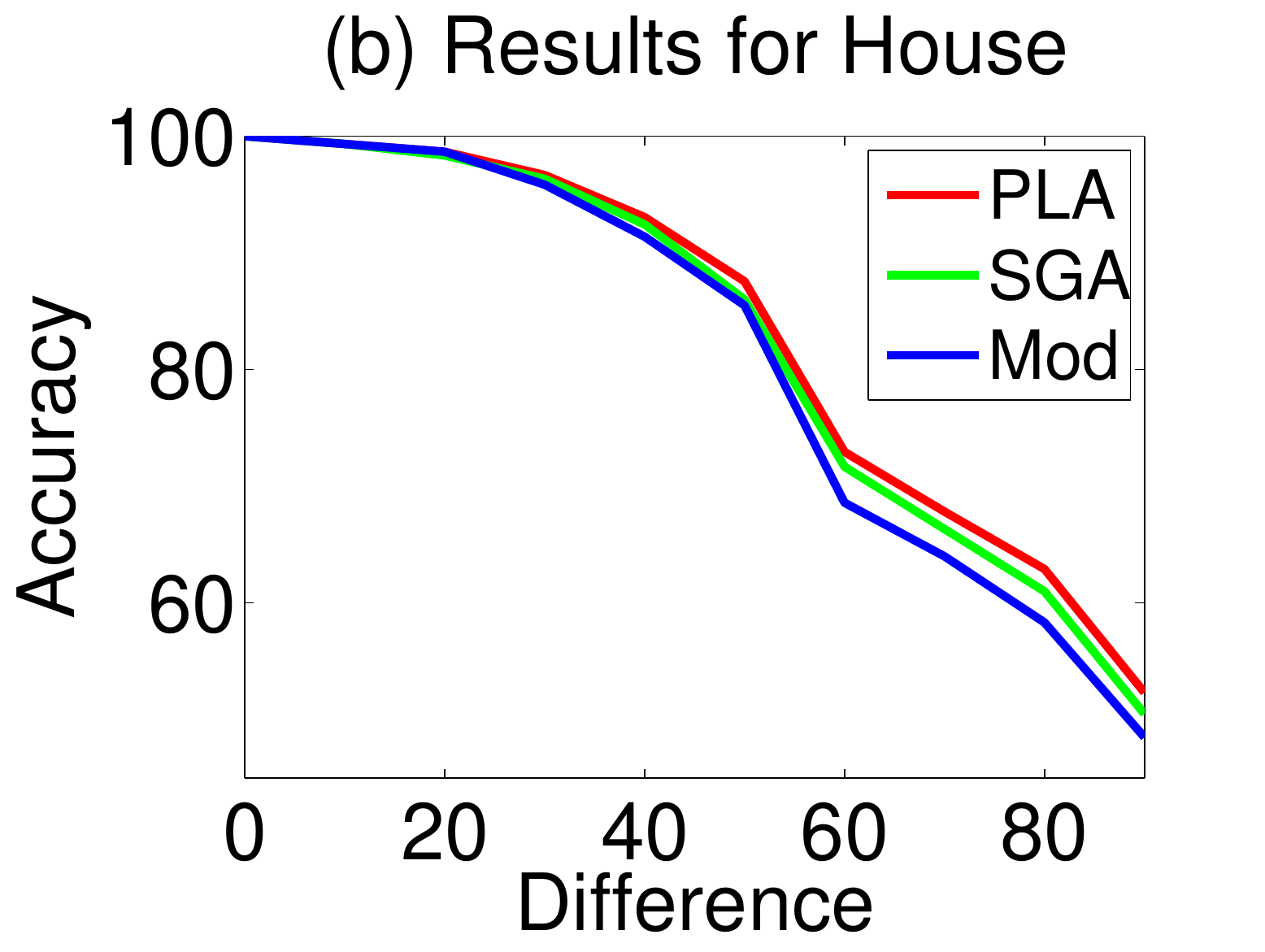}
        ~           \includegraphics[width=0.2\textwidth]{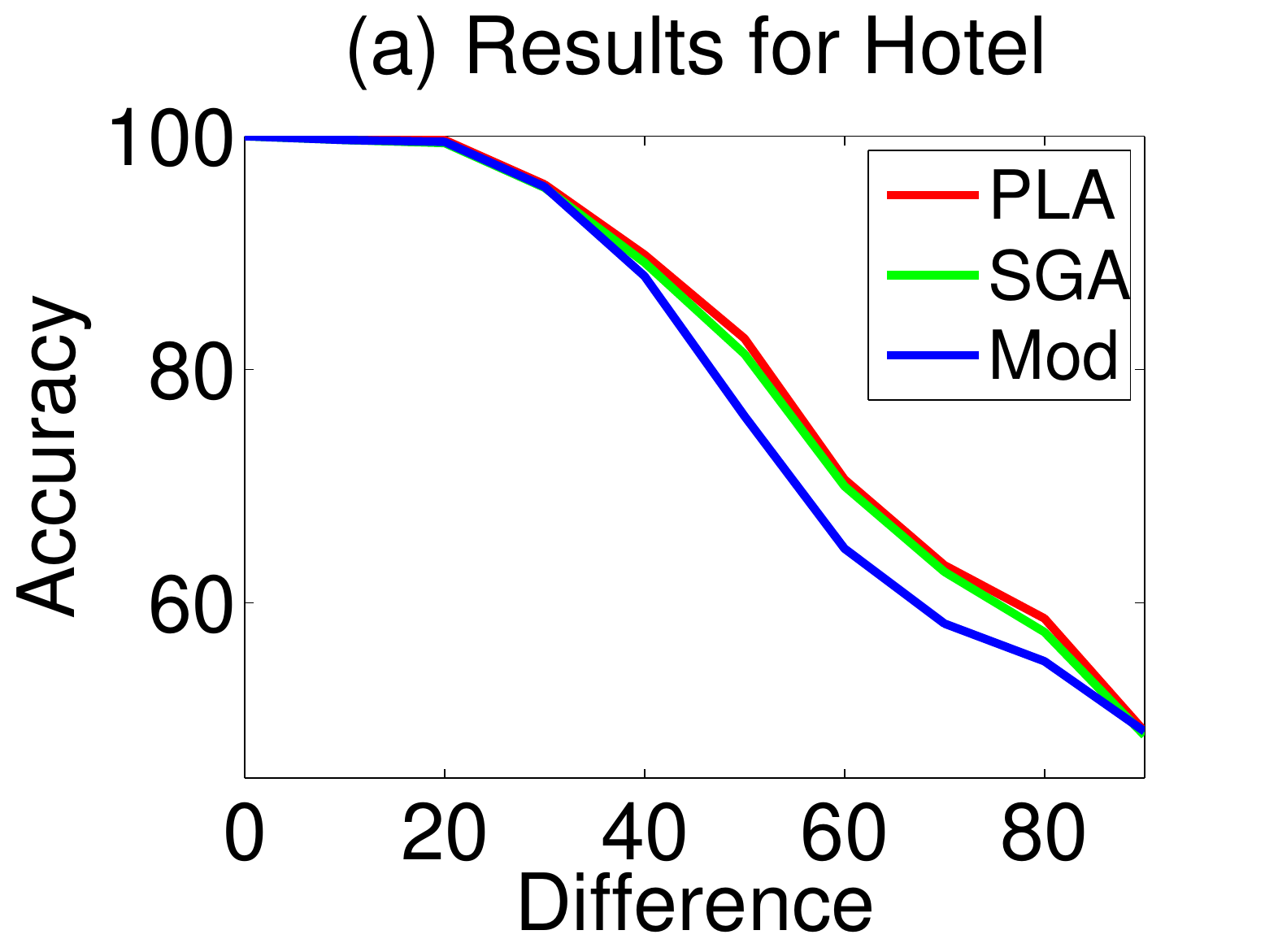}
 \caption{(a) (far left) shows the clustering used in cooperative matching, (b) and (c) show the results with bipartite matching and cooperative matching respectively, and (d) and (e) (far right) give the results on the \href{http://vasc.ri.cmu.edu//idb/html/motion/house/}{House} and \href{http://vasc.ri.cmu.edu//idb/html/motion/hotel/}{Hotel} dataset, 
        showing PLA (this paper)
achieves slightly better than SGA,
and both submodular methods performing better than standard matching (Mod).}
\label{fig:match}
\end{figure*}

The problem of matching key-points in images, also called image
correspondence, is an important problem in computer
vision~\cite{ogale2005shape}. The simplest model for this problem
constructs a matching with linear scores, i.e., a max bipartite
matching~\cite{jegelka2013interactive}, called a \emph{linear
  assignment}. This model does not allow a representation of
interaction between the pixels. For example, we see many obviously
spurious matches in figure~\ref{fig:match}b. Many models try to
capture this, via, for example via quadratic
assignments~\cite{caetano2009learning}. Instead of just looking at the
best linear assignment, the quadratic models try to incorporate
pairwise constraints. This is also called graph matching.


We describe a new and different model here.  First, we cluster
key-points, separately in each of the two images, into $k$
clusters. Figure~\ref{fig:match}a shows a particular clustering of an
image into $k = 3$ groups. The clustering can be performed based on
the pixel color map, or simply the distance of the key-points. That is,
each image has $k$ clusters. Let $\{ V_i^{(1)} \}_{i=1}^k$ and $\{
V_i^{(2)} \}_{i=1}^k$ be the two sets of clusters.  We then compute
the linear assignment problem, letting $\mathcal M \subseteq \mathcal
E$ be the resulting maximum matching. We then partition the edge set $\mathcal
E = \mathcal E_1 \cup \mathcal E_2 \cup \dots \mathcal E_k \cup
\mathcal E'$ where $\mathcal E_i = \mathcal M \cap (V_\ell^{(1)}
\times V_s^{(2)} )$ for $\ell,s \in \{ 1, 2, \dots, k\}$ corresponding
to the $i$'th largest intersection, and $\mathcal E' = \{\mathcal E
\backslash \cup_{i = 1}^k \mathcal E_i\}$ are the remaining edges
either that were not matched or that did not lie within a frequently
associated pair of image key-point clusters. We then
define a submodular function as follows:
\begin{align}\label{coopobj}
f(S) = \sum_{i = 1}^k \psi_i(w(S \cap \mathcal E_i)) + 
w(S \cap \mathcal E'), 
\end{align}
which provides an additional discount to the edges $\{\mathcal
E_i\}_{i = 1}^k$ corresponding to key-points that were frequently
associated in the initial pass. The problem of co-operative matching then becomes an instance of Problem 1 with the submodular function (over the edges) defined above, and a constraint that the edges form a matching. Figures~\ref{fig:match}b and
\ref{fig:match}c shows how the submodular matchings improve over 
the simple bipartite matching, with $k = 3$. The minimum matching
approach obtains many spurious matches between clusters (shown in
red), while the cooperation described above reduces these spurious
matches. The cooperative matching improves the performance over the
modular method on these images by about $20\%$.

We also test the performance of our algorithms on the CMU House and Hotel dataset~\cite{caetano2009learning}. The house dataset has $111$ images, while the hotel dataset has $101$ images. We consider all possible pairs of images, with differences between the two images ranging from $0:10:90$ in both cases. We consider three algorithms: PLA, SGA (both using Equation~\eqref{coopobj}) and the simple modular bipartite matching as a baseline (Mod). Again, we set $k = 3$. 
The results are shown in Figure~\ref{fig:match}(d-e) where 
we observe that PLA and SGA beat Mod by about $3-5\%$ on average. 
Moreover, we also see that PLA, in general, outperforms SGA, thus showing how superior theoretical guarantees translate into better empirical performance. In PLA, we chose $\epsilon$ such that each concave function $\psi_i$ has four break points. We observed, moreover, that setting lower values of $\epsilon$ does not improve the objective value in this application. 
\JTR{Should definitely give the $\epsilon$ used, and also report running times. Ideally, we would have multiple plots for various different $\epsilon$s and
that also show running time as well as accuracy.}
We observe, moreover, that PLA also beats SGA in terms of objective value. We do not compare the ellipsoidal approximation algorithm (EA)~\cite{goemans2009approximating}, mainly because it is too slow to run on real world problems. Moreover, this algorithm has been observed to perform comparably to the much simpler SGA~\cite{rkiyersemiframework2013}. While we considered the simple linear assignment as a baseline for the cooperative matching, it seems possible to embed this cooperation on more involved graph matching models as well. 

\subsection{Sensor Placement}
\label{sec:sensor-placement}

\begin{figure*}
\centering           \includegraphics[width=0.18\textwidth]{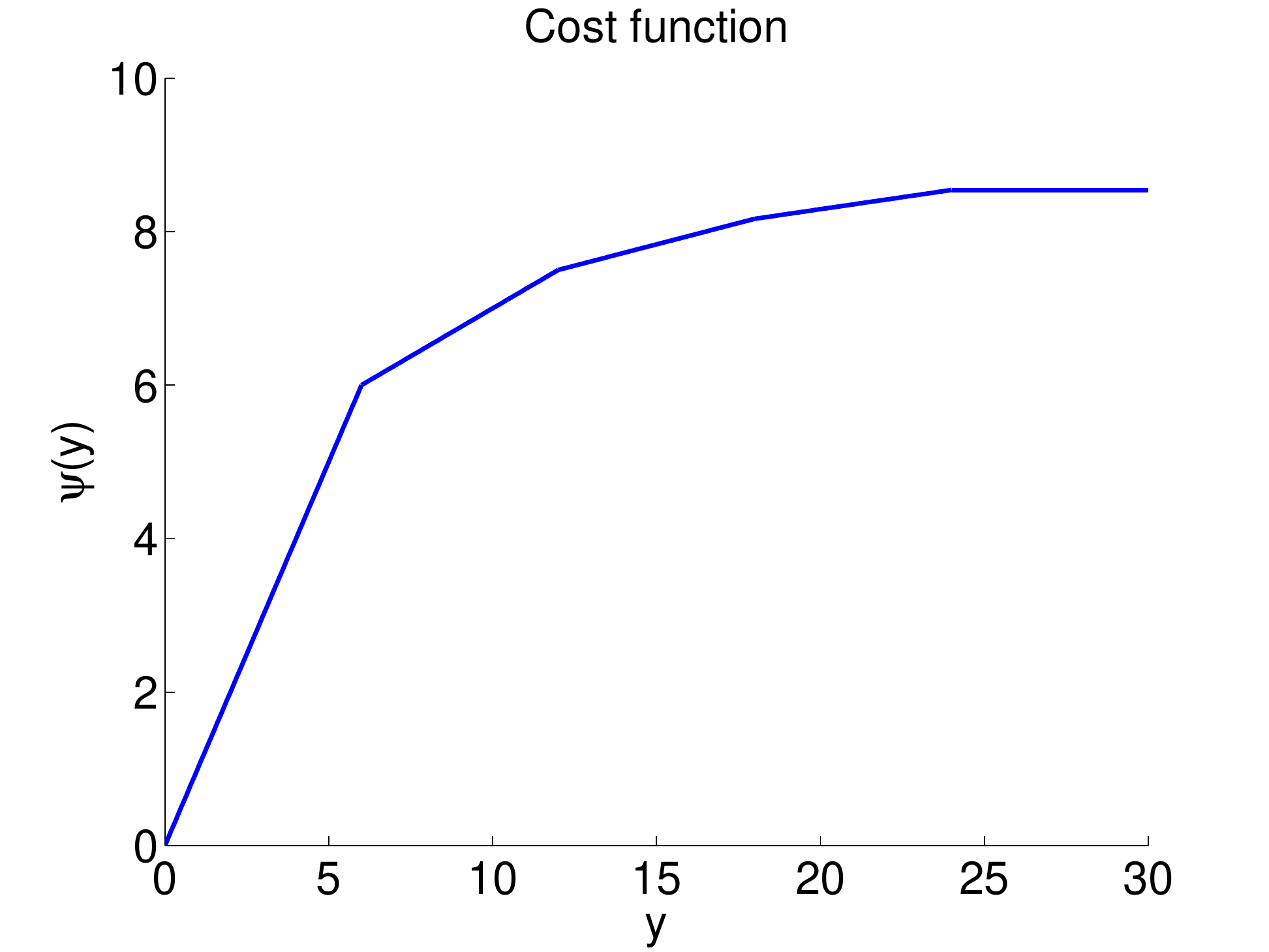}
        ~             \includegraphics[width=0.18\textwidth]{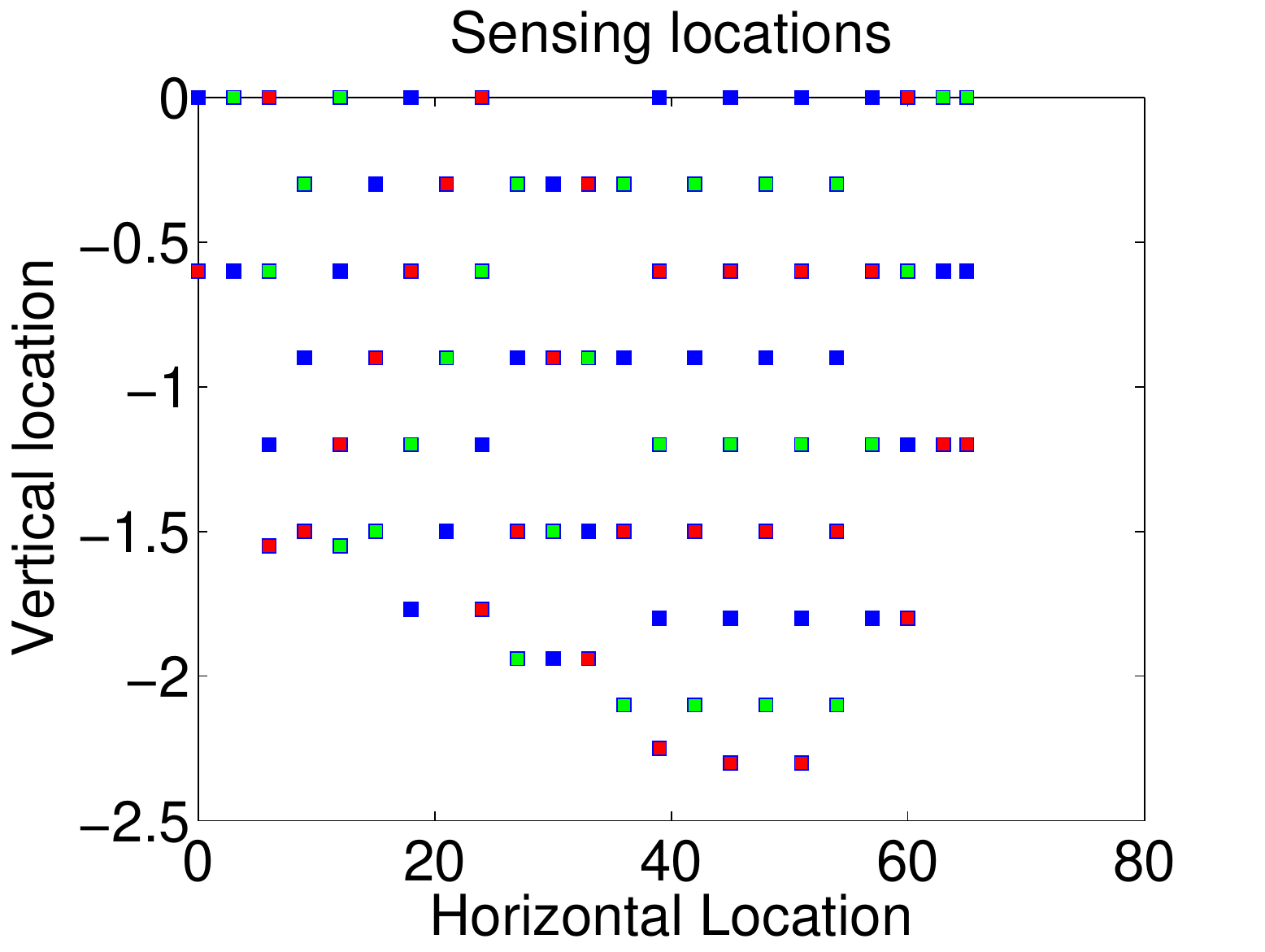}
        ~                \includegraphics[width=0.18\textwidth]{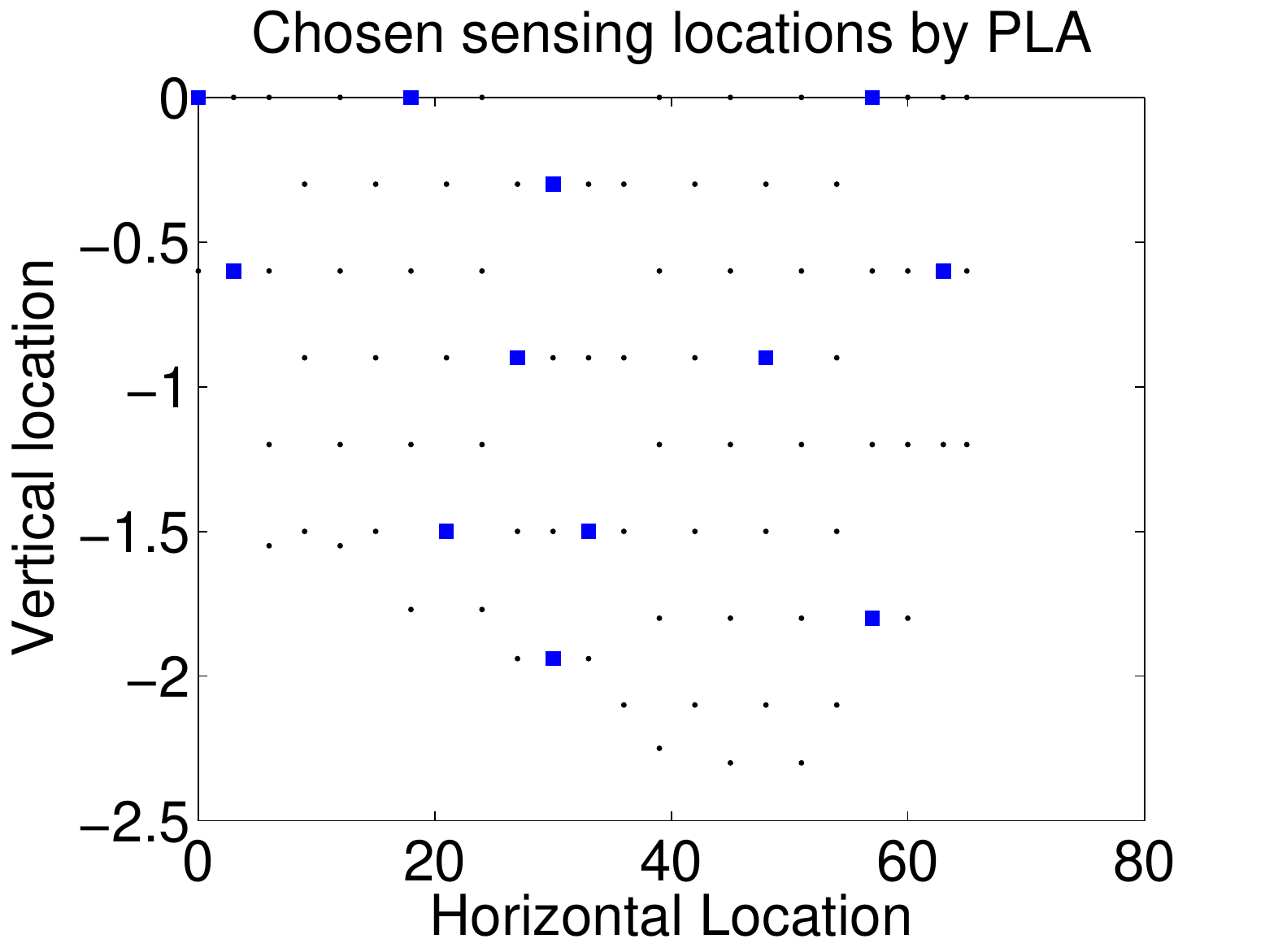}
        ~                \includegraphics[width=0.18\textwidth]{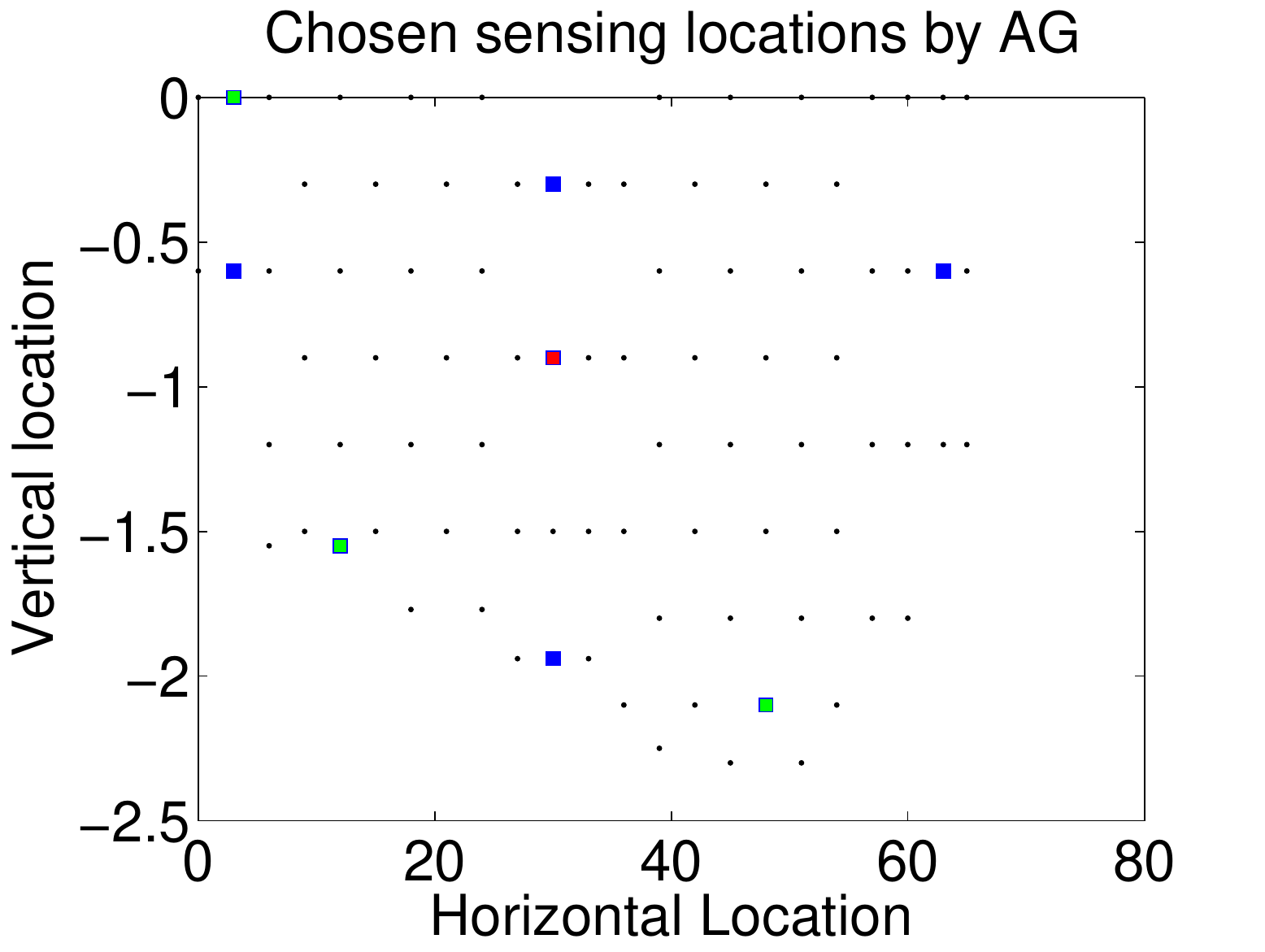}
        ~               \includegraphics[width=0.18\textwidth]{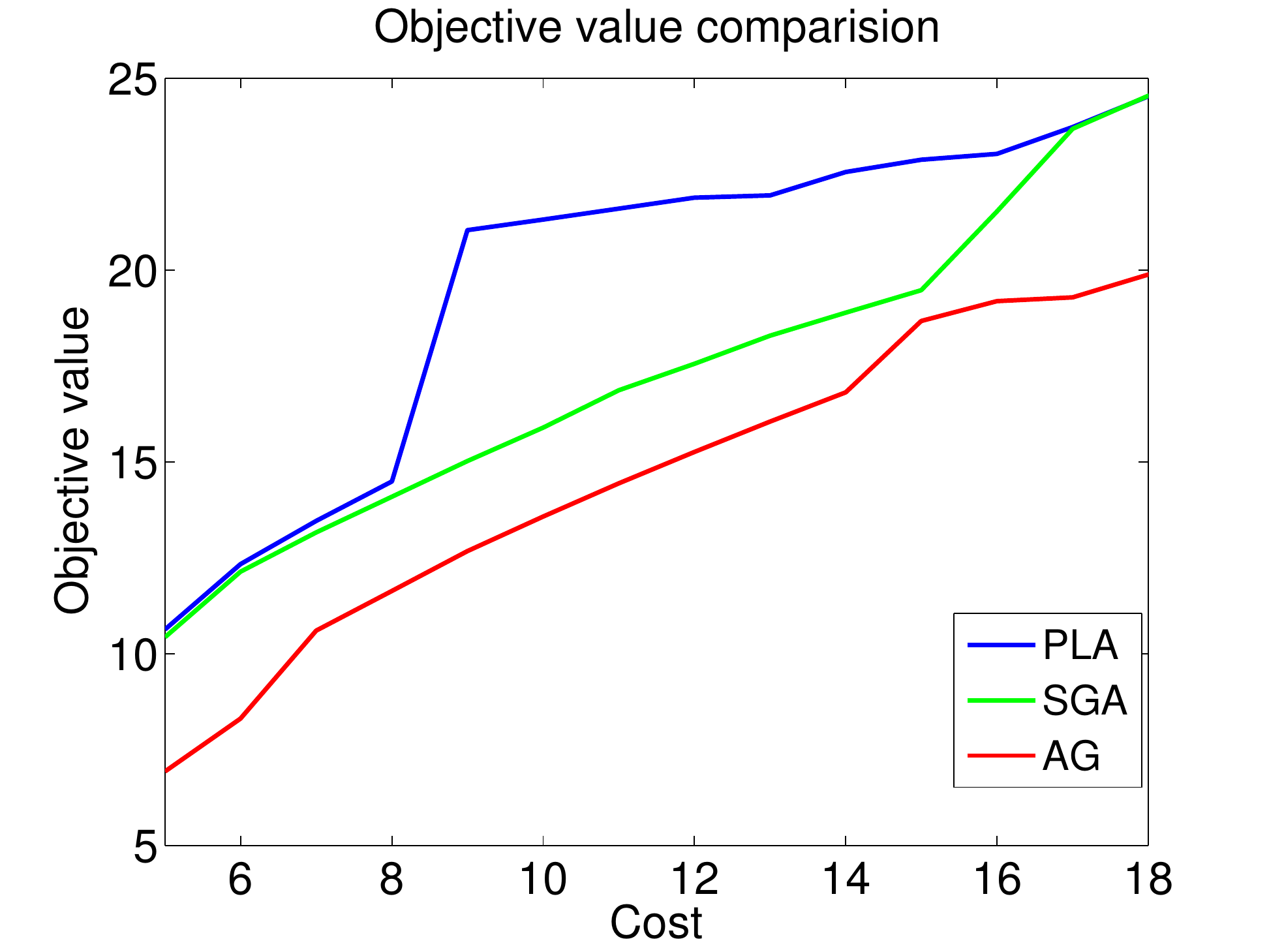}
        \caption{(a) Cost function $\psi_i(y)$ used, (b) set of locations with the three colors (blue, red, green) referring to the three types of sensors, (c) Sensors chosen by PLA, (d) Sensors chosen by AG and (e) plots of the objective value of different algorithms (please zoom in for details).}\label{fig:sensorplace}
\end{figure*}

We next consider an application of sensor placement. A number of natural models for this problem are forms of submodular maximization~\cite{krause2009optimizing, krause2008near}. A natural model, that performs very well in practice, is to maximize the mutual information $I(X_A; X_{V \backslash A})$, where $A$ refers to the set of sensors chosen. \cite{krause2009optimizing, krause2008near} investigate this in the setting of additive costs on the sensors. Often however, the costs are not additive in practice. In fact, very often, they are also submodular~\cite{krause2008near}, and a natural model is,
\begin{align}\label{costfn}
f(X) = \sum_{i=1}^k \psi_i( c(X \cap S_i) )
\end{align}
 where $\psi_i$s are concave, $c(j)$ is the cost of sensor $j$ and $S_1, \cdots, S_k$ are groups of similar sensors. This was posed as an open problem in \cite{krause2008near}. We can naturally pose this as instances of Problem 3, where $g(X) = I(X_A; X_{V \backslash A})$ and $f(X)$ is the cost function above. Note that we could equivalently also express this as an instance of Problem 2 with a constraint on $g$ while minimizing $f$.
 
We consider real world data of placing sensors to predict the pH
values from the lake of Merced~\cite{krause2009optimizing}. We also
assume that the function $f$ is piece-wise linear, shown in
Figure~\ref{fig:sensorplace}a (far left).  
Figure~\ref{fig:sensorplace}b shows the locations (horizontal and vertical). We assume that there are three kinds of locations, shown in blue, green and red colors respectively, and the costs of placing sensors in the same kind of location is discounted. Correspondingly, we assume the cost function is an instance of function Equation~\eqref{costfn} with $k = 3$. For simplicity, we assume also that all three types of sensor locations have the same coverage model (though, in general, it would make sense for them to have different models for coverage, based on their type). Under this assumption, the optimal configuration would tend to be spatially diverse, yet cooperative (in the sense, that the same type of sensors would be chosen). 
\JTR{TODO: This section needs to be fixed. We only have one type of sensor,
but multiple types of locations. The types of locations interact (i.e., underwater,
mountain, on a tower) in that it is cheaper to place sensors
in similar types of locations (due to shared fixed costs of purchasing
equipment, human capital costs and shared fixed finding and hiring costs of personnel, economies of scale of equipment, etc.). Also, the above right plot should
be changed to ``cost'' rather than budget.}

We compare three algorithms: PLA, and SGA (both on Problem 3), and a simple cost agnostic greedy algorithm (AG), which ignores the cost function $f$, and greedily adds sensors. Figure~\ref{fig:sensorplace}c shows the sensors chosen by PLA (the cost sensitive one), and Figure~\ref{fig:sensorplace}d shows the choices of AG (the cost agnostic one). While both have the same cost budget, the cost agnostic one does not utilize the discounts of placing sensors in similar locations, and correspondingly, places fewer sensors. The cost sensitive algorithms (PLA and SGA) on the other hand, simultaneously achieve coverage, while making use of the discounts. Figure~\ref{fig:sensorplace}e plots the objective functions attained by the three algorithms. We see that both PLA and SGA, outperform the agnostic greedy algorithm. Moreover, PLA also performs better than SGA. Note that the function $f$, used in this case is piece-wise linear, and correspondingly PLA is exact in this case.

\section{Conclusions}
\label{sec:conclusions}

In this paper, we investigated a new class of algorithms for various
forms of constrained submodular programs, with a special subclass of
submodular cost functions. We focus on problems that for the general
class of submodular functions are hard, and yet occur naturally in
many applications. We showed that when we restrict the class of
functions to low rank sums of concave over modular functions, we can
obtain significantly improved worst case theoretical results. We also
complemented our results with experimental results in sensor placement
and image correspondence. An immediate open question is whether there
are similar algorithms for other rich and useful subclasses of
submodular functions. In particular, it would be interesting if one
can remove the low rank assumption, and provide tighter approximation
algorithms for general sums of concave over modular functions, which
would be very powerful.

This material is based upon work supported by the National Science
Foundation under Grant No. (IIS-1162606), as well as a Google and a Microsoft
award. This work was also supported in part by the CONIX Research Center,
one of six centers in JUMP, a Semiconductor Research Corporation (SRC)
program sponsored by DARPA.

\bibliographystyle{abbrv}
\bibliography{submod}
\end{document}